\documentclass[sigconf]{acmart}
\usepackage{multirow}
\usepackage{subfigure}
\usepackage{balance}
\newcommand{\m}{{\sf {EigenGCN}}}
\newcommand{\pooling}{{\sf {EigenPooling}}}
\usepackage{balance}
\def\BibTeX{{\rm B\kern-.05em{\sc i\kern-.025em b}\kern-.08emT\kern-.1667em\lower.7ex\hbox{E}\kern-.125emX}}

\begin{document}

%
\title{Graph Convolutional Networks with EigenPooling}

\author{Yao Ma}

\affiliation{%
	\institution{Michigan State University}
}
\email{mayao4@msu.edu}

\author{Suhang Wang}

\affiliation{%
	\institution{Pennsylvania State University}
}
\email{szw494@psu.edu}

\author{Charu C. Aggarwal}

\affiliation{%
	\institution{IBM T. J. Watson Research Center}
}
\email{charu@us.ibm.com}

\author{Jiliang Tang}

\affiliation{%
	\institution{Michigan State University}
}
\email{tangjili@msu.edu}

\begin{abstract}
Graph neural networks, which generalize deep neural network models to graph structured data, have attracted increasing attention in recent years. They usually learn node representations by transforming, propagating and aggregating node features and have been proven to improve the performance of many graph related tasks such as node classification and link prediction. To apply graph neural networks for the graph classification task, approaches to generate the \textit{graph representation} from node representations are demanded. A common way is to globally combine the node representations. However, rich structural information is overlooked. Thus a hierarchical pooling procedure is desired to preserve the graph structure during the graph representation learning. There are some recent works on hierarchically learning graph representation analogous to the pooling step in conventional convolutional neural (CNN) networks. However, the local structural information is still largely neglected during the pooling process. In this paper, we introduce a pooling operator $\pooling$ based on graph Fourier transform, which can utilize the node features and local structures during the pooling process. We then design pooling layers based on the pooling operator, which are further combined with traditional GCN convolutional layers to form a graph neural network framework $\m$ for graph classification. Theoretical analysis is provided to understand $\pooling$ from both local and global perspectives. Experimental results of the graph classification task on $6$ commonly used benchmarks demonstrate the effectiveness of the proposed framework. 
\end{abstract}

%
%

%

%

%

\maketitle

\section{Introduction}

Recent years have witnessed increasing interests in generalizing neural networks for graph structured data. The stream of research on this topic is usually under the name of ``Graph Neural Networks''~\cite{scarselli2009graph}, which typically involves transforming, propagating and aggregating node features across the graph. Among them, some focus on node-level representation learning~\cite{kipf2016semi,hamilton2017inductive,schlichtkrull2018modeling} while others investigate learning graph-level representation~\cite{li2015gated,henaff2015deep,duvenaud2015convolutional,defferrard2016convolutional,bruna2013spectral,ying2018hierarchical,gao2019graph1,gao2019graph2}. While standing from different perspectives, these methods have been proven to advance various graph related tasks. The methods focusing on node representation learning have brought improvement to tasks such as node classification~\cite{kipf2016semi,hamilton2017inductive,schlichtkrull2018modeling,gao2018large,gao2019graph1,gao2019graph2} and link prediction~\cite{schlichtkrull2018modeling} and those methods working on graph-level representation learning have mainly facilitated graph classification. In this paper, {\it we work on graph level representation learning with a focus on the task of graph classification.} 
 
The task of graph classification is to predict the label of a given graph utilizing its associated features and graph structure. Graph Neural Networks can extract graph representation while using all associated information. Majority of existing graph neural networks~\cite{dai2016discriminative,duvenaud2015convolutional,gilmer2017neural,li2015gated} have been designed to generate good node representations, and then globally summarize the node representations as the graph representation. These methods are inherently ``flat'' since they treat all the nodes equivalently when generating graph representation using the node representations. In other words, the entire graph structure information is totally neglected during this process. However, nodes are naturally of different statuses and roles in a graph, and they should contribute differently to the graph level representation. Furthermore, graphs often have different local structures (or subgraphs), which contain vital graph characteristics. For instance, in a graph of a protein, atoms (nodes) are connected via bonds (edges); some local structures, which consist of groups of atoms and their direct bonds, can represent some specific functional units, which, in turn, are important to tell the functionality of the entire protein~\cite{shervashidze2011weisfeiler,duvenaud2015convolutional,borgwardt2005protein}. These local structures are also not captured during the global summarizing process. To generate the graph representation which preserves the local and global graph structures, a hierarchical pooling process, analogous to the pooling process in conventional convolutional neural (CNN) networks~\cite{krizhevsky2012imagenet}, is needed.  
 
 There are very recent works investigating the pooling procedure for graph neural networks~\cite{ying2018hierarchical,defferrard2016convolutional,fey2018splinecnn,simonovsky2017dynamic}. These methods group nodes into subgraphs (supernodes), coarsen the graph based on these subgraphs and then the entire graph information is reduced to the coarsened graph by generating features of supernodes from their corresponding nodes in subgraphs. However, when pooling the features for supernodes, average pooling or max pooling have been usually adopted where the structures of these group nodes (the local structures) are still neglected. With the local structures, the nodes in the subgraphs are of different statuses and roles when they contribute to the supernode representations. It is challenging to design a general pooling operator while incorporating the local structure information as 1) the subgraphs may contain different numbers of nodes, thus a fixed size pooling operator cannot work for all subgraphs; and 2) the subgraphs could have very different structures, which may require different approaches to summarize the information for the supernode representation. To address the aforementioned challenges, we design a novel pooling operator $\pooling$ based on the eigenvectors of the subgraphs, which naturally have the same size of each subgraph and can effectively capture the local structures when summarizing node features for supernodes. $\pooling$ can be used as pooling layers to stack with any graph neural network layers to form a novel framework $\m$ for graph classification. Our major contributions can be summarized as follows:
 
 \begin{itemize}
     \item We introduce a novel pooling operator $\pooling$, which can naturally summarize the subgraph information while utilizing the subgraph structure;
     \item We provide theoretical understandings on $\pooling$ from both local and global perspectives;
     \item We incorporate pooling layers based on $\pooling$ into existing graph neural networks as a novel framework $\m$ for representation learning for graph classification; and 
     \item We conduct comprehensive experiments on numerous real-world graph classification benchmarks to demonstrate the effectiveness of the proposed pooling operator.
 \end{itemize} 
 

\section{The Proposed Framework -- $\m$} \label{sec:proposed} 

In this paper, we aim to develop a Graph Neural Networks (GNN) model, which consists of convolutional layers and pooling layers, to learn graph representations such that graph level classification can be applied. Before going to the details, we first introduced some notations and the problem setting. 

\vspace{0.5em}
\noindent{}\textbf{Problem Setting:} A graph can be represented as $\mathcal{G} =\{\mathcal{E},\mathcal{V}\}$, where $\mathcal{V} = \{v_1,\dots, v_N\}$ is the set of $N$ nodes and $\mathcal{E}$ is the set of edges. The graph structure information can also be represented by an adjacency matrix ${\bf A}\in \mathbb{R}^{N\times N}$. Furthermore, each node in the graph is associated with node features and we use ${\bf X}\in \mathbb{R}^{N\times d}$ to denote the node feature matrix, where $d$ is the dimension of features. Note that this node feature matrix can also be viewed as a $d$-dimensional graph signal~\cite{shuman2013emerging} defined on the graph $\mathcal{G}$. In the graph classification setting, we have a set of graphs $\{\mathcal{G}_i\}$, each graph $\mathcal{G}_i$ is associated with a label $y_i$. The task of the graph classification is to take the graph (structure information and node features) as input and predict its corresponding label. To make the prediction, it is important to extract useful information from both graph structure and node features. We aim to design graph convolution layers and $\pooling$ to hierarchically extract graph features, which finally learns a vector representation of the input graph for graph classification.

\subsection{An Overview of \m}
In this work, we build our model based on Graph Convolutional Networks (GCN)~\cite{kipf2016semi}, which has been demonstrated to be effective in node-level representation learning. While the GCN model is originally designed for semi-supervised node classification, we only discuss the part for node representation learning but ignoring the classification part. The GCN is stacked by several convolutional layers and a single convolutional layer can be written as:
\begin{align}
    {\bf F}^{i+1} = ReLU(\tilde{\bf D}^{-\frac{1}{2}}\tilde{\bf A} \tilde{\bf D}^{-\frac{1}{2}} {\bf F}^{i} {\bf W}^{i})
\end{align}
where ${\bf F}^{i+1}\in \mathbb{R}^{N\times d_{i+1}}$ is the output of the $i$-th convolutional layer for $i>0$ and ${\bf F}^{0} = {\bf X}$ denotes the input node features. A total number of $I$ convolutional layers are stacked to learn node representations and the output matrix ${\bf F}^{I}$ can be viewed as the final node representations learned by the GCN model. 

As we described above, the GCN model has been designed for learning node representations. In the end, the output of the GCN model is a matrix instead of a vector. The procedure of the GCN is rather ``flat'', as it can only ``pass message'' between nodes through edges but cannot summarize the node information into the higher level graph representation. A simple way to summarize the node information to generate graph level representation is global pooling. For example, we could use the average of the node representations as the graph representation. However, in this way, a lot of key information is ignored and the graph structure is also totally overlooked during the pooling process. 

To address this challenge, we propose eigenvector based pooling layers $\pooling$ to hierarchically summarize node information and generate graph representation. An illustrative example is demonstrated in Figure~\ref{fig:general_framework}. In particular, several pooling layers are added between convolutional layers. Each of the pooling layers pools the graph signal defined on a graph into a graph signal defined on a coarsened version of the input graph, which consists of fewer nodes. Thus, the design of the pooling layers consists of two components: 1) graph coarsening, which divides the graph into a set of subgraphs and form a coarsened graph by treating subgraphs as supernodes; and 2) transform the original graph signal information into the graph signal defined on the coarsened graph with $\pooling$. We coarsen the graph based on a subgraph partition. Given a subgraph partition with no overlaps between subgraphs, we treat each of the subgraphs as a supernode. To form a coarsened graph of the supernodes, we determine the connectivity between the supernodes by the edges across the subgraphs. During the pooling process, for each of the subgraphs, we summarize the information of the graph signal on the subgraph to the supernode. With graph coarsening, we utilize the graph structure information to form coarsened graphs, which makes it possible to learn representations level by level in a hierarchical way. With $\pooling$, we can learn node features of the coarsened graph that exploits the subgraph structure as well as the node features of the input graph.

\begin{figure}
    \centering
    \includegraphics[scale=0.3]{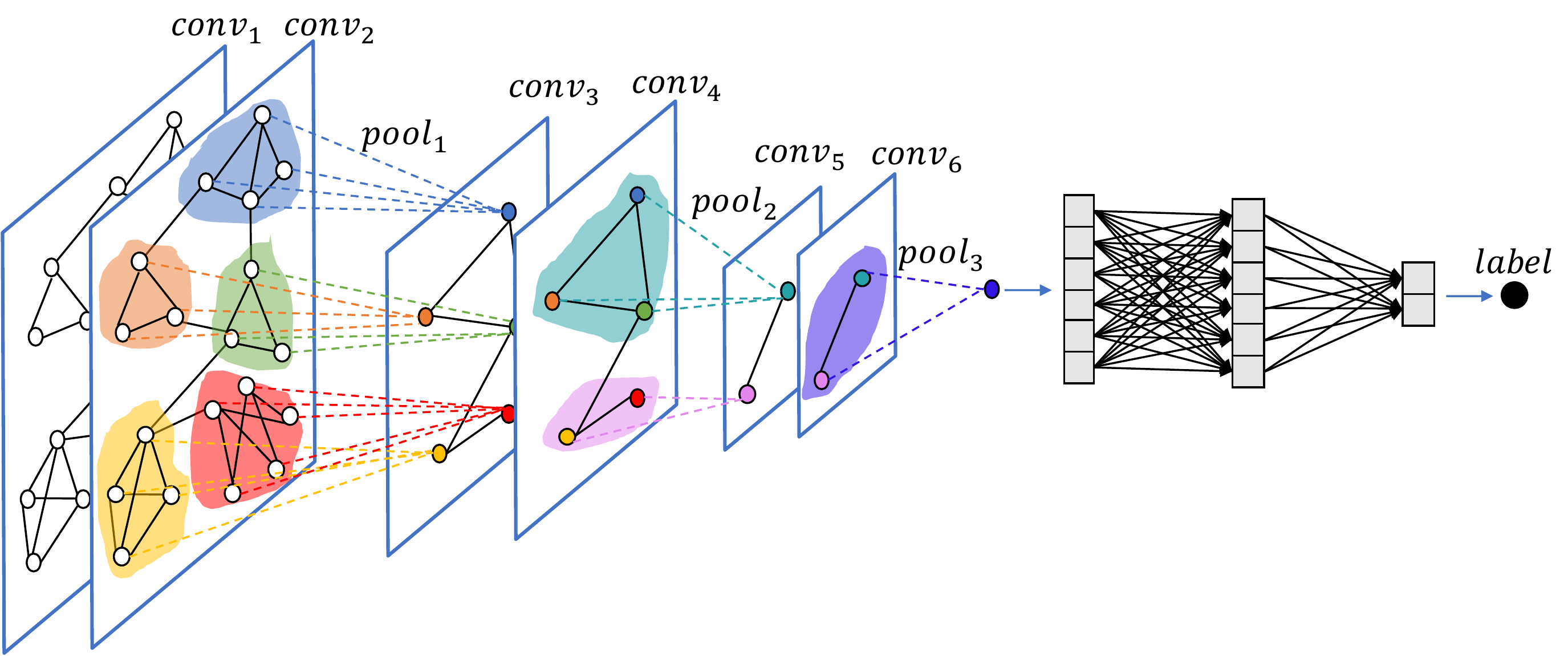}
    \vskip -1em
    \caption{An illustrative example of the general framework}
    \label{fig:general_framework}
\end{figure}
Figure~\ref{fig:general_framework} shows an illustrative example, where a binary graph classification is performed. In this illustrative example, the graph is coarsened three times and finally becomes a single supernode. The input is a graph signal (the node features), which can be multi-dimensional. For the ease of illustration, we do not show the node features on the graph. Two convolutional layers are applied to the graph signal. Then, the graph signal is pooled to a signal defined on the coarsened graph. This procedure (two convolution layers and one pooling layer) is repeated two more times and the graph signal is finally pooled to a signal on a single node. This pooled signal on the single node, which is a vector, can be viewed as the graph representation. The graph representation then goes through several fully connected layers and the prediction is made upon the output of the last layer. Next, we introduce details of graph coarsening and $\pooling$ of \m. 

\subsection{Graph Coarsening}
In this subsection, we introduce how we perform the graph coarsening. As we mentioned in the previous subsection, the coarsening process is based on subgraph partition. There are different ways to separate a given graph to a set of subgraphs with no overlapping nodes. In this paper, we adopt spectral clustering to obtain the subgraphs, so that we can control the number of the subgraphs, which, in turn, determines the pooling ratio. We leave other options as future work. Given a set of subgraphs, we treat them as supernodes and build the connections between them as similar in~\cite{tremblay2016subgraph}. An example of the graph coarsening and supernodes is shown in Figure~\ref{fig:general_framework}, where a subgraph and its supernodes are denoted using the same color. Next, we introduce how to mathematically describe the subgraphs, supernodes, and their relations. 

Let ${\bf c}$ be a partition of a graph $\mathcal{G}$, which consists of $K$ connected subgraphs $\{\mathcal{G}^{(k)}\}_{k=1}^K.$ For the graph $\mathcal{G}$, we have the adjacency matrix $\mathbf{A}\in \mathbb{R}^{N\times N}$ and the feature matrix $\mathbf{X}\in \mathbb{R}^{N\times d}$. Let $N_k$ denote the number of nodes in the subgraph $\mathcal{G}^{(k)}$ and $\Gamma^{(k)}$ is the list of nodes in subgraph $\mathcal{G}^{(k)}$. Note that each of the subgraph can be also viewed as a supernode. For each subgraph $\mathcal{G}^{(k)}$, we can define a sampling operator $\mathbf{C}^{(k)} \in  \mathbb{R}^{N\times N_k}$ as follows:
\begin{align}
    \mathbf{C}^{(k)}[i,j] = 1 \quad \text{if and only if} \quad \Gamma^{(k)}(j) = v_i,
\end{align}
where ${\bf C}^{(k)}[i,j]$ denotes the element in the $(i,j)$-th position of ${\bf C}^{(k)}[i,j]$ and $\Gamma^{(k)}(j)$ is the $j$-th element in the node list $\Gamma^{(k)}$. This operator provides a relation between nodes in the subgraph $\mathcal{G}^{(k)}$ and the nodes in the original graph. Given a single dimensional graph signal ${\bf x} \in \mathbb{R}^{N\times 1}$ defined on the original entire graph, the induced signal that is only defined on the subgraph $\mathcal{G}^{(k)}$ can be written as  
\begin{equation}
    {\bf x}^{(k)} = ({\bf C}^{(k)})^T {\bf x}.
\end{equation}
On the other hand, we can also use $\mathbf{C}^{(k)}$ to up-sample a graph signal ${\bf x}^{(k)}$ defined only on the subgraph $\mathcal{G}^{(k)}$ to the entire graph $\mathcal{G}$ by 
\begin{equation}
    {\bar{\bf x}}=\mathbf{C}^{(k)} {\bf x}^{(k)}.
\end{equation}
It keeps the values of the nodes in the subgraph untouched while setting the values of all the other nodes that do not belong to the subgraph to $0$. The operator can be applied to multi-dimensional signal ${\bf X}\in {\mathcal{R}^{N\times d}}$ in a similar way. The induced adjacency matrix ${\bf A}^{(k)} \in \mathbb{R}^{N_k\times N_k}$ of the subgraph $\mathcal{G}^{(k)}$, which only describes the connection within the subgraph $\mathcal{G}^{(k)}$, can be obtained as
\begin{align}
    \mathbf{A}^{(k)} = (\mathbf{C}^{(k)} )^T \mathbf{A}  \mathbf{C}^{(k)}.
\end{align}
The intra-subgraph adjacency matrix of the graph $\mathcal{G}$, which only consists of the edges inside each subgraph, can be represented as
\begin{align}
    \mathbf{A}_{int} = \sum\limits_{k=1}^{K} {\bf C}^{(k)} {\bf A}^{(k)}  ({\bf C}^{(k)})^T.
\end{align}
Then the inter-subgraph adjacency matrix of graph $\mathcal{G}$, which only consists of the edges between subgraphs, can be represented as ${\bf A}_{ext}={\bf A} - {\bf A}_{int}$. 

Let  $\mathcal{G}_{coar}$ denote the coarsened graph, which consists of the supernodes and their connections. We define the assignment matrix ${\bf S} \in \mathbb{R}^{N\times K}$, which indicates whether a node belongs to a specific subgraph as:
\begin{align*}
    {\bf S}[i,j] = 1 \quad \text{if and only if } v_i \in \Gamma^{(j)}.
\end{align*}
Then, the adjacency matrix of the coarsened graph is given as 
\begin{align}
    {\bf A}_{coar} = {\bf S}^T {\bf A}_{ext} {\bf S}. 
\end{align}

With Graph Coarsening, we can obtain the connectivity of $\mathcal{G}_{coar}$, i.e., ${\bf A}_{coar}$. Obviously, ${\bf A}_{coar}$ encodes the network structure information of $\mathcal{G}$. Next, we describe how to obtain the node features $\mathbf{X}_{coar}$ of $\mathcal{G}_{coar}$ using $\pooling$. With ${\bf A}_{coar}$ and $\mathbf{X}_{coar}$, we can stack more layers of GCN-GraphCoarsening-$\pooling$ to learn higher level representations of the graph for classification.

\subsection{Eigenvector-Based Pooling -- $\pooling$}

In this subsection, we introduce $\pooling$, aiming to obtain $\mathbf{X}_{coar}$ that encodes network structure information and node features of $\mathcal{G}$. Globally, the pooling operation is to transform a graph signal defined on a given graph to a corresponding graph signal defined on the coarsened version of this graph. It is expected that the important information of the original graph signal can be largely preserved in the transformed graph signal. Locally, for each of the subgraph, we summarize the features of the nodes in this subgraph to a single representation of the supernode. It is necessary to consider the structure of the subgraph when we perform the summarizing, as the subgraph structure also encodes important information. However, common adopted pooling methods such as max pooling~\cite{ying2018hierarchical,defferrard2016convolutional} or average pooling~\cite{duvenaud2015convolutional} ignored the graph structure. In some works~\cite{niepert2016learning}, the subgraph structure is used to find a canonical ordering of the nodes, which is, however, very difficult and expensive. In this work, to use the structure of the subgraphs, we design the pooling operator based on the graph spectral theory by facilitating the eigenvectors of the Laplacian matrix of the subgraph. Next, we first briefly review the graph Fourier transform and then introduce the design of $\pooling$ based on graph Fourier transform. 

\subsubsection{Graph Fourier Transform}\label{sec:fourier}
Given a graph $\mathcal{G} = \{\mathcal{E},\mathcal{V}\}$ with $\mathbf{A}\in \mathbb{R}^{N\times N}$ being the adjacency matrix and ${\bf X}\in \mathbb{R}^{N\times d}$ being the node feature matrix. Without loss of generality, for the following description, we consider $d=1$, i.e., $\mathbf{x} \in \mathbb{R}^{N\times 1}$, which can be viewed as a single dimensional graph signal defined on the graph $\mathcal{G}$~\cite{sandryhaila2013discrete}. This is the spatial view of a graph signal, which maps each node in the graph to a scalar value (or a vector if the graph signal is multi-dimensional). Analogous to the classical signal processing, we can define graph Fourier transform~\cite{shuman2013emerging} and spectral representation of the graph signal in the spectral domain. To define the graph signal in the spectral domain, we need to use the Laplacian matrix~\cite{chung1997spectral} $\mathbf{L} = \mathbf{D} - \mathbf{A}$, where $\mathbf{D}$ is the diagonal degree matrix with $\mathbf{D}[i,i] = \sum\limits_{j=1}^N A[i,j]$. The Laplacian matrix ${\bf L}$ can be used to define the ``smoothness'' of a graph signal~\cite{shuman2013emerging} as follows:
\begin{align}
    s({\bf x}) =  {\bf x}^T {\bf L} {\bf x} = \frac{1}{2} \sum\limits_{i,j}^{N} {\bf A}[i,j] ({\bf x}[i]- {\bf x}[j])^2.
\end{align}
$s({\bf x})$ measures the smoothness of the graph signal ${\bf x}$. The smoothness of a graph signal depends on how dramatically the value of connected nodes can change. The smaller $s({\bf x})$, the more smooth it is. For example, for a connected graph, a graph signal with the same value on all the nodes has a smoothness of $0$, which means ``extremely smooth'' with no change. 

As $\mathbf{L}$ is a real symmetric semi-positive definite matrix, it has a completed set of orthonormal eigenvectors $\{u_l\}_{l=1}^{N}$. These eigenvectors are also known as the graph Fourier modes~\cite{shuman2013emerging}, which are associated with the ordered real non-negative eigenvalues $\{\lambda_l\}_{l}^{N}$. Given a graph signal ${\bf x}$, the graph Fourier transform can be obtained as follows
\begin{align}
    \hat{\bf x} = \mathbf{U^T x},
\end{align}
where $\mathbf{U}=[u_1,\dots, u_{N}] \in \mathbb{R}^{N\times N}$ is the matrix consists of the eigenvectors of $\mathbf{L}$. The vector $\hat{\bf x}$ obtained after the transform is the representation of the graph signal in the spectral domain. Correspondingly, the inverse graph Fourier transform, which transfers the spectral representation back to the spatial representation, can be denoted as: 
\begin{align} \label{eq:rec}
    \mathbf{x} = \mathbf{U}\hat{\bf x}.
\end{align}

Note that we can also view each the eigenvector $u_l$ of the Laplacian matrix $\mathbf{L}$ as a graph signal, and its corresponding eigenvalue $\lambda_l$ can measure its ``smoothness''. For any of the eigenvector $u_l$, we have: 
\begin{align}
    s({\bf u}_l) = {\bf u}_l ^T {\bf L} {\bf u}_l = {\bf u}_l^T \lambda_l {\bf u}_l = \lambda_l.
\end{align}
The eigenvectors (or Fourier modes) are a set of base signals with different ``smoothness'' defined on the graph $\mathcal{G}$. Thus, the graph Fourier transform of a graph signal $\bf{x}$ can be also viewed as linearly decomposing the signal $\bf x$ into the set of base signals. $\hat{\bf x}$ can be viewed as the coefficients of the linear combination of the base signals to obtain the original signal ${\bf x}$. 

\subsubsection{The Design of Pooling Operators}
Since graph Fourier transform can transform graph signal to spectral domain which takes into consideration both graph structure and graph signal information, we adopt graph Fourier transform to design pooling operators, which pool the graph signal defined on a given graph $\mathcal{G}$ to a signal defined on its coarsened version $\mathcal{G}_{coar}$. The design of the pooling operator is based on graph Fourier transform of the subgraphs $\{\mathcal{G}^{k}\}_{k=1}^K$. Let $\mathbf{L}^{(k)}$ denote the Laplacian matrix of the subgraph $\mathcal{G}^{(k)}$. We denote the eigenvectors of the Laplacian matrix ${\bf L}^{(k)}$ as $u^{(k)}_1,\dots, u^{(k)}_{N_k}$. We then use the up-sampling operator $C^{(k)}$ to up-sample these eigenvectors (base signals on this subgraph) into the entire graph and get the up-sampled version as:
\begin{align} \label{eq:bar_u_lk}
   \bar{\bf u}^{(k)}_l = {\bf C}^{(k)} {\bf u}^{(k)}_l, l=1\dots N_k.
\end{align}
With the up-sampled eigenvectors, we organize them into matrices to form pooling operators. Let $\Theta_l \in \mathbb{R}^{N\times K}$ denote the pooling operator consisting of all the $l$-th eigenvectors from all the subgraphs
\begin{align} \label{eq:Theta_l}
    \Theta_l = [\bar{u}^{(1)}_l,\dots,\bar{u}^{(K)}_l]
\end{align}

Note that the subgraphs are not necessary all with the same number of nodes, which means that the number of eigenvectors can be different. Let $N_{max} = \max\limits_{k=1,\dots,K} N_k$ be the largest number of nodes among all the subgraphs. Then, for a subgraph $\mathcal{G}^{(k)}$ with $N_k$ nodes, we set $\mathbf{u}^{(k)}_{l}$ for $N_k < l \leq  N_{max}$ as ${\bf 0} \in \mathbb{R}^{N_k\times 1}$. The pooling process with $l$-th pooling operator $\Theta_l$ can be described as
\begin{align} \label{eq:X_l}
    {\bf X}_l = \Theta_l^T{\bf X} 
\end{align}
where ${\bf X}_l\in \mathbb{R}^{K\times d}$ is the pooled result using the $l$-th pooling operator. The $k$-th row of ${\bf X}_l$ contains the information pooled from the $k$-th subgraph, which is the representation of the $k$-th supernode. 

Following this construction, we build a set of $N_{max}$ pooling operators. To combine the information pooled by different pool operators, we can concatenate them together as follows: 

\begin{align}
    {\bf X}_{pooled} = [{\bf X_0}, \dots,{\bf X}_{N_{max}}].
\end{align}
where ${\bf X}_{pooled} \in \mathbb{R}^{K\times d\cdot N_{max}}$ is the final pooled results. 
For efficient computation, instead of using the results pooled by all the pooling operators, we can choose to only use the first $H$ of them as follows: 
\begin{align}
    {\bf X}_{coar}={\bf X}_{pooled} = [{\bf X_0}, \dots,{\bf X}_{H}].
\end{align}
In Section~\ref{sec:local} and Section~\ref{sec:global}, we will show that with $H \ll N_{max}$, we can still preserve most of the information. We will further empirically investigate the effect of choice of $H$ in Section~\ref{sec:experimental_results}



\section{Theoretical Analysis of $\pooling$} \label{sec:theoretical}
In this section, we provide a theoretical analysis of $\pooling$ by understanding it from local and global perspectives. 
We prove that the pooling operation can preserve useful information to be processed by the following GCN layers. We also verify that $\m$ is permutation invariant, which lays the foundation for graph classification with $\m$.

\subsection{A Local View of \pooling}
\label{sec:local}

In this subsection, we analyze the pooling operator from a local perspective focusing on a specific subgraph $\mathcal{G}^{(k)}$. For the subgraph $\mathcal{G}^{(k)}$, the pooling operator tries to summarize the nodes' features and form a representation for the corresponding supernode of the subgraph. For a pooling operator $\Theta_l$, the part that is effective on the subgraph $\mathcal{G}^{(k)}$, is only the up-sampled eigenvector $\bar{\bf u}^{(k)}_l$ as the other eigenvectors have $0$ values on the subgraph $\mathcal{G}^{(k)}$. Without the loss of generality, let's consider a single dimensional graph signal ${\bf x} \in \mathbb{R}^{N\times 1}$ defined on the $\mathcal{G}$, the pooling operation on $\mathcal{G}^{(k)}$ can be represented as:
\begin{align}
    (\bar{\bf u}^{(k)}_l )^T  {\bf x} = ({\bf u}^{(k)}_l)^T{\bf x}^{(k)},
\end{align}
which is the Fourier coefficient of the Fourier mode ${\bf u}^{(k)}_l$ of the subgraph $\mathcal{G}^{(k)}$. Thus, from a local perspective, the pooling process is a graph Fourier transform of the graph signal defined on the subgraph. As we introduced in the Section~\ref{sec:fourier}, each of the Fourier modes (eigenvectors) is associated with an eigenvalue, which measures its smoothness. The Fourier coefficient of the corresponding Fourier mode provides the information to indicate the importance of this Fourier mode to the signal. The coefficient summarizes the graph signal information utilizing both the node features and the subgraph structure as the smoothness is related to both of them. Each of the coefficients characterizes a different property (smoothness) of the graph signal. Using the first $H$ coefficients while discarding the others means that we focus more on the ``smoother'' part of the graph signal, which is common in a lot of applications such as signal denoising and compression~\cite{tremblay2016subgraph,chen2014signal,narang2012perfect}. Therefore, we can use the squared summation of the coefficients to measure how much information can be preserved as shown in the following theorem.  

\begin{theorem}
Let ${\bf x}$ be a graph signal defined on the graph $\mathcal{G}$ and $\mathbf{U} = [{\bf u}_1, \dots, {\bf u}_{N}]$ be the Fourier modes of this graph. Let $\hat{\bf x}$ be the corresponding Fourier coefficients, i.e., $\hat{\bf x} = \mathbf{U}^T \mathbf{x}$. Let ${\bf x}' = \sum\limits_{l=1}^{H}\hat{\bf x}[l]\cdot {\bf u}_l$ be the signal re-constructed using only the first $H$ Fourier modes. Then $\frac{||\hat{\bf x}[1:H]||^2_2}{||\hat{\bf x}||^2_2}$ can measure the information being preserved by this re-construction. Here $\hat{\bf x}[1:H]$ denotes the vector consisting of the first $H$ elements of $\hat{\bf x}$.

\end{theorem}
\begin{proof}
According to Eq.(\ref{eq:rec}), ${\bf x}$ can be written as ${\bf x} = \sum\limits_{l=1}^{N} \hat{{\bf x}}[l] \cdot {\bf u}_l$. Since $\mathbf{U}$ is orthogonal, we have
\begin{align}
    \frac{\|{\bf x}'\|_2^2}{\|{\bf x}\|_2^2} = \frac{\| \sum\limits_{l=1}^{H} \hat{\bf x}[l]\cdot {\bf u}_l\|^2_2}{\|\sum\limits_{l=1}^{N} \hat{{\bf x}}[l] \cdot {\bf u}_l\|_2^2} = \frac{\|\hat{\bf x}[1:H]\|^2_2}{\|\hat{\bf x}\|^2_2}
\end{align}
which completes the proof.
\end{proof}

It is common that for natural graph signal that the magnitude of the spectral form of the graph signal is concentrated on the first few coefficients~\cite{sandryhaila2013discrete,shuman2013emerging}, which means that $\frac{\|\hat{\bf x}[1:H]\|^2_2}{\|\hat{\bf x}\|^2_2} \approx 1$ for $H\ll N_k$. In other words, by using the first $H$ coefficients, we can \textit{preserve the majority of the information} while \textit{reducing the computational cost}. We will empirically verify it in the experiment section.

\subsection{A Global View of \pooling} \label{sec:global}
In this subsection, we analyze the pooling operators from a global perspective focusing on the entire graph $\mathcal{G}$. The pooling operators we constructed can be viewed as a filterbank~\cite{tremblay2016subgraph}. Each of the filters in the filterbank filters the given graph signal and obtains a new graph signal. In our case, the filtered signal is defined on the coarsened graph $\mathcal{G}_{coar}$. Without the loss of generality, we consider a single dimensional signal ${\bf x}\in \mathbb{R}^{N\times 1}$ of $\mathcal{G}$, then the filtered signals are $\{{\bf x}_l\}_{l=1}^{N_{max}}$. Next, we describe some key properties of these pooling operators. 

\vspace{0.5em}
\noindent{}\textbf{Property 1: Perfect Reconstruction:} The first property is that when $N_{max}$ number of filters are used, the input graph signal can be perfectly reconstructed from the filtered signals.
\begin{lemma} \label{lemma_1}
The graph signal ${\bf x}$ can be perfectly reconstructed from its filtered signals $\{{\bf x}_l\}_{l=1}^{N_{max}}$ together with the pooling operators $\{ \Theta_l\}_{l=1}^{N_{max}}$ as $\mathbf{x} = \sum\limits_{l=1}^{N_{max}}  \Theta_l {\bf x}_l$.
\end{lemma}
\vspace{-1em}
\begin{proof}
With the definition of $\Theta_l$ given in Eq.(\ref{eq:Theta_l}), we have
\begin{equation} \label{eq:sum_theta_x}
  \sum\limits_{l=1}^{N_{max}}  \Theta_l {\bf x}_l = \sum\limits_{l=1}^{N_{max}} \sum\limits_{k=0}^K \bar{\bf u}^{(k)}_l \cdot {\bf x}_l[k] = \sum\limits_{k=0}^K \sum\limits_{l=1}^{N_{max}}  \bar{\bf u}^{(k)}_l \cdot {\bf x}_l[k]
\end{equation}
From Eq.(\ref{eq:X_l}), we know that $\mathbf{x}_l[k] = ({\bar{\mathbf{u}}_l^{(k)}})^T \mathbf{x}$. Together with the fact that $\bar{\mathbf{u}}_l^{(k)} = \mathbf{C}^{(k)}\mathbf{u}_l^{(k)}$ in Eq.(\ref{eq:bar_u_lk}), we can rewrite  $\sum\limits_{l=1}^{N_{max}}  \bar{\bf u}^{(k)}_l \cdot {\bf x}_l[k]$ as 
\begin{equation}
    \sum\limits_{l=1}^{N_{max}}  \bar{\bf u}^{(k)}_l \cdot {\bf x}_l[k] =  \mathbf{C}^{(k)} (\sum_{l=1}^{N_{max}} \mathbf{u}_l^{(k)} {\mathbf{u}_l^{(k)}}^T) {\mathbf{C}^{(k)}}^T \mathbf{x}
\end{equation}
Obviously, $\sum_{l=1}^{N_{max}} \mathbf{u}_l^{(k)} {\mathbf{u}_l^{(k)}}^T = \sum_{l=1}^{N_{K}} \mathbf{u}_l^{(k)} {\mathbf{u}_l^{(k)}}^T = \mathbf{I}$, since that $\{ \mathbf{u}_l^{(k)} \}_{l=1}^{N_k}$ are orthonormal and $\{\mathbf{u}_l^{(k)}\}_{l=N_k+1}^{N_{max}}$ are all $\mathbf{0}$ vectors. Thus, $   \sum\limits_{l=1}^{N_{max}}  \bar{\bf u}^{(k)}_l \cdot {\bf x}_l[k] = \mathbf{C}^{(k)} \mathbf{x}^{(k)}$. Substitute this to Eq.(\ref{eq:sum_theta_x}), we arrive at
\begin{equation}
    \sum\limits_{l=1}^{N_{max}}  \Theta_l {\bf x}_l = \sum_{k=1}^{K} \mathbf{C}^{(k)} \mathbf{x}^{(k)} = \mathbf{x}
\end{equation}
which completes the proof.

\end{proof}
From Lemma~\ref{lemma_1}, we know if $N_{max}$ number of filters are chosen, the filtered signals $\{ {\bf x}_l\}_{l=1}^{N_{max} }$ can preserve all the information from $\mathbf{x}$. Thus, together with graph coarsening, eigenvector pooling can preserve the signal information of the input graph and can enlarge the receptive filed, which allows us to finally learn a vector representation for graph classification. 

\vspace{0.5em}
\noindent{}\textbf{Property 2: Energy/Information Preserving} The second property is that the filtered signals preserve all the energy when $N_{max}$ filters are chosen. To show this, we first give the following lemma, which serves as a tool for demonstrating property 2.
\begin{lemma} \label{lemma:orthogonal}
All the columns in the operators $\{ \Theta_l\}_{l=1}^{N_{max}}$ are orthogonal to each other.
\end{lemma}
\vspace{-1em}
\begin{proof}
    By definition, we know that, for the same $k$, i.e, the same subgraph, ${\bf u}^{(k)}_l, l=1,\dots N_{max}$ are orthogonal to each other, which means $\bar{\bf u}^{(k)}_l, l=1,\dots N_{max}$ are also orthogonal to each other. In addition, all the $\bar{\bf u}^{(k)}_l$ with different $k$ are also orthogonal to each other as they only have non-zero values on different subgraphs. 
\end{proof}

With the above lemma, we can further conclude that the $\ell_2$ norm of graph signal $\mathbf{x}$ is equal to the summation of the $\ell_2$ norm of the pooled signals $\{\mathbf{x}_l\}_{l=1}^{N_{max}}$. The proof is given as follows:
\begin{lemma} \label{lemma:l_2 norm}
The $\ell_2$ norm of the graph signal ${\bf x}$ is equal to the summation of the $\ell_2$ norm of the pooled signals $\{{\bf x}_l\}_{l=1}^{N_{max}}$:
\begin{align}
  ||{\bf x}||_2^2  = \sum\limits_{l=1}^{N_{max}} ||{\bf x}_l||^2_2
\end{align}
\end{lemma}\label{lem:reconstruction}
\vspace{-1em}
\begin{proof}
From Lemma~\ref{lemma_1} and \ref{lemma:orthogonal}, we have
\begin{align}
\|{\bf x}\|_2^2 &=  \| \sum\limits_{l=1}^{N_{max}} \Theta_l {\bf x}_l \|_2^2  = \left|\left|\sum\limits_{k=0}^K \sum\limits_{l=1}^{N_{max}}  \bar{\bf u}^{(k)}_l \cdot {\bf x}_l[k]\right|\right|^2_2 \nonumber \\
 & = \sum\limits_{k=0}^K \sum\limits_{l=1}^{N_{max}} {\bf x}_l^2[k] = \sum\limits_{l=1}^{N_{max}} ||{\bf x}_l||^2_2 \nonumber
\end{align}
which completes the proof.
\end{proof}

\vspace{0.5em}
\noindent{}\textbf{Property 3: Approximate Energy Preserving} Lemma~\ref{lemma:l_2 norm} describes the energy preserving when $N_{max}$ number of filters are chosen. In practice, we only need $H \ll N_{max}$ of filters for efficient computation. Next we show that even with $H$ number of filters, the filtered signals preserve most of the energy/information.
\begin{theorem}
Let ${\bf x}' = \sum\limits_{l=1}^{H}  \Theta_l {\bf x}_l$ be the graph signal reconstructed only using the first $H$ pooling operators and pooled signals $\{{\bf x}_l\}_{l=1}^{H}$. Then the ratio $\frac{\sum\limits_{l=1}^{N_{H}} ||{\bf x}_l||^2_2}{\sum\limits_{l=1}^{N_{max}} ||{\bf x}_l||^2_2}$ can measure the portion of information being preserved by this ${\bf x}'$. 
\label{thm:energy_preserving_global}
\end{theorem}
\vspace{-1em}
\begin{proof}
As shown in Lemma \ref{lem:reconstruction}, $\|{\bf x}\|^2_2 = \sum\limits_{l=1}^{N_{max}} \|{\bf x}_l\|^2_2$. Similarly, we can show that $\|{\bf x}'\|^2_2 = \sum\limits_{l=1}^{H} \|{\bf x}_l\|^2_2$. The portion of the information being preserved can be represented as 
\begin{align} \label{thm:approximate_energy_preserving}
    \frac{\|{\bf x}'\|^2_2}{\|{\bf x}\|^2_2} = \frac{\sum\limits_{l=1}^{N_{H}} ||{\bf x}_l||^2_2}{\sum\limits_{l=1}^{N_{max}} ||{\bf x}_l||^2_2}.
\end{align}
which completes the proof.
\end{proof}
Since for natural graph signals, the magnitude of the spectral form of the graph signal is concentrated in the first few coefficients~\cite{sandryhaila2013discrete}, which means that even with $H$ filters, EigenPooling can preserve the majority of the information/energy.

\subsection{Permutation Invariance of \m}
$\m$ takes the adjacency matrix $\mathbf{A}$ and the node feature matrix $\mathbf{X}$ as input and aims to learn a vector representation of the graph. The nodes in a graph do not have a specific ordering, i.e., $\mathbf{A}$ and $\mathbf{X}$ can be permuted. Obviously, for the same graph, we want $\m$ to extract the same representation no matter which permutation of $\mathbf{A}$ and $\mathbf{X}$ are used as input. Thus, in this subsection, we prove that $\m$ is permutation invariant, which lays the foundation of using $\m$ for graph classification.
\begin{theorem}
    Let $\mathbf{P} = \{0,1\}^{n \times n}$ be any permutation matrix, then $\m(\mathbf{A}, \mathbf{X}) = \m(\mathbf{P}\mathbf{A}\mathbf{P}^T, \mathbf{P}\mathbf{X})$, i.e., $\m$ is permutation invariant.
\end{theorem}
\vspace{-1em}
\begin{proof}
    In order to prove that $\m$ is permutation invariant, we only need to show that it's key components GCN, graph coarsening and EigenPooling are permutation invariant. For GCN, before permutation, the output is ${\bf F} = ReLU(\tilde{\bf D}^{-\frac{1}{2}}\tilde{\bf A} \tilde{\bf D}^{-\frac{1}{2}} {\bf X} {\bf W}^{i})$. With permutation, the output becomes
    \begin{equation}
        ReLU\Big((\mathbf{P}\tilde{\bf D}^{-\frac{1}{2}}\tilde{\bf A} \tilde{\bf D}^{-\frac{1}{2}} \mathbf{P}^T) (\mathbf{P}{\bf X}) {\bf W}\Big) = ReLU(\mathbf{P}\tilde{\bf D}^{-\frac{1}{2}}\tilde{\bf A} \tilde{\bf D}^{-\frac{1}{2}} {\bf X} {\bf W}) = \mathbf{P} \mathbf{F} \nonumber
    \end{equation}
where we have used $\mathbf{P}^T \mathbf{P} = \mathbf{I}$. This shows that the effect of permutation on GCN only permutes the order of the node representations but doesn't change the value of the representations. Second, the graph coarsening is done by spectral clustering with $\mathbf{A}$. No matter which permutation we have, the detected subgraphs will not change. Finally, EigenPooling summarizes the information within each subgraph. Since the subgraph structures are not affected by the permutation and the representation of each node in the subgraphs is also not affected by the permutation, we can see that the supernodes' representations after EigenPooling are not affected by the permutation. In addition, the inter-connectivity of supernodes is not affected since it's determined by spectral clustering. Thus, we can say that one step of GCN-Graph Coarsening-EigenPooling is permutation invariant. Since finally $\m$ learns one vector representation of the input graph, we can conclude that the vector representation is the same under any permutation $\mathbf{P}$. 
\end{proof}

\section{Experiment} \label{sec:experimental_results}
In this section, we conduct experiments to evaluate the effectiveness of the proposed framework $\m$. Specifically, we aim to answer two questions:
\begin{itemize}
    \item Can $\m$ improve the graph classification performance by the design of $\pooling$?
    \item How reliable it is to use $H$ number of filters for pooling?
\end{itemize}
We begin by introducing datasets and experimental settings. We then compare $\m$ with representative and state-of-the-art baselines for graph classification to answer the first question. We further conduct analysis on graph signals to verify the reasonableness of using $H$ filters, which answers the second question.

\subsection{Data sets}

To verify whether the proposed framework can hierarchically learn good graph representations for classification, we evaluate $\m$ on 6 widely used benchmark data sets for graph classification~\cite{KKMMN2016}, which includes three protein graph data sets, i.e., {\bf ENZYMES}~\cite{borgwardt2005protein,schomburg2004brenda}, {\bf PROTEINS}~\cite{borgwardt2005protein,dobson2003distinguishing}, and ${\bf D\& D}$~\cite{dobson2003distinguishing,shervashidze2011weisfeiler}; one mutagen data set {\bf Mutagenicity}~\cite{riesen2008iam,kazius2005derivation} (We denoted as MUTAG in Table~\ref{tab:statistics} and Table~\ref{tab:performance}); and two data sets that consist of chemical compounds screened for activity against non-small cell lung cancer and ovarian cancer cell lines, {\bf NCI1} and {\bf NCI109}~\cite{wale2008comparison}. 
Some statistics of the data sets can be found in Table~\ref{tab:statistics}. From the table, we can see that the used data sets contain varied numbers of graphs and have different graph sizes. We include data sets of different domains, sample and graph sizes
to give a comprehensive understanding of how $\m$ performs with data sets under various conditions.
{\footnotesize
\begin{table}[h]
	\begin{center}
		\vspace{-0.15in}
		\caption{Statistics of datasets\label{tab:statistics}}
		\vspace*{-0.15in}
		\begin{tabular}{c |c|c |c|c|c|c} 
			\hline
			&ENZYMES & PROTEINS& D\& D & NCI1 & NCI109 & MUTAG \\ 
			\hline
			$\#$ graphs &600  & 1,113 & 1,178 & 4,110&4,127& 4,337\\
			\hline
			 mean |V| & 32.63& 39.06 & 284.32 & 29.87 &29.68&30.32\\
			\hline
			$\#$ classes & 6 & 2&2 &2 &2& 2\\
			\hline 
		\end{tabular}
			\vspace{-0.15in}
	\end{center}
\end{table}
}
\subsection{Baselines and Experimental Settings}

To compare the performance of graph classification, we consider some representative and state-of-the-art graph neural network models with various pooling layers. Next, we briefly introduce these baseline approaches as well as the experimental settings for them.

\begin{itemize}
    \item {\bf GCN}~\cite{kipf2016semi} is a graph neural network framework proposed for semi-supervised node classification. It learns node representations by aggregating information from neighbors. As the {\bf GCN} model does not consist of a pooling layer, we directly pool the learned node representations as the graph representation. We use it as a baseline to compare whether a hierarchical pooling layer is necessary. 
    \item {\bf GraphSage}~\cite{hamilton2017inductive} is similar as the GCN and provides various aggregation method. As similar in {\bf GCN}, we directly pool the learned node representations as the graph representation. 
    \item {\bf SET2SET}. This baseline is also built upon {\bf GCN}, it is also ``flat'' but uses set2set architecture introduced in~\cite{vinyals2015order} instead of averaging over all the nodes. We select this method to further show whether a hierarchical pooling layer is necessary no matter average or other pooling methods are used. 
    \item {\bf DGCNN}~\cite{zhang2018end} is built upon the GCN layer. The features of nodes are sorted before feeding them into traditional 1-D convolutional and dense layers~\cite{zhang2018end}. This method is also ``flat'' without a hierarchical pooling procedure. 
    \item {\bf Diff-pool}~\cite{ying2018hierarchical} is a graph neural network model designed for graph level representation learning with differential pooling layers. It uses node representations learned by an additional convolutional layer to learn the subgraphs (supernodes) and coarsen the graph based on it. We select this model as it achieves state-of-art performance on the graph classification task.
    \item \m-H represents various variants of the proposed framework \m, where $H$ denotes the number of pooling operators we use for $\pooling$. In this evaluation, we choose $H=1,2,3$.
\end{itemize}

For each of the data sets, we randomly split it to $3$ parts, i.e., $80\%$ as training set, $10\%$ as validation set and $10\%$ as testing set. We repeat the randomly splitting process $10$ times, and the average performance of the $10$ different splits are reported. The parameters of baselines are chosen based on their performance on the validate set.  For the proposed framework, we use the $9$ splits of the training set and validation set to tune the structure of the graph neural network as well as the learning rate. The same structure and learning rate are then used for all $9$ splits. 

Following previous work~\cite{ying2018hierarchical}, we adopt the widely used evaluation metric, i.e., Accuracy, for graph classification to evaluate the performance.

\subsection{Performance on Graph Classification}
Each experiment is run 10 times and the average graph classification performance in terms of accuracy is reported in Table~\ref{tab:performance}. From the table, We make the following observations:
\begin{itemize}
    \item Diff-pool and the $\m$ framework perform better than those methods without a hierarchical pooling procedure in most of the cases. Aggregating the node information hierarchically can help learn better graph representations.
    \item The proposed framework $\m$ shares the same convolutional layer with GCN, GraphSage, and SET2SET. However, the proposed framework (with different $H$) outperforms them in most of the data sets. This further indicates the necessity of the hierarchical pooling procedure. In other words, the proposed $\pooling$ can indeed help the graph classification performance.  
    \item In most of the data sets, we can observe that the variants of the $\m$ with more eigenvectors achieve better performance than those with fewer eigenvectors. Including more eigenvectors, which suggests that we can preserve more information during pooling, can help learn better graph representations in most of the cases. In some of data sets, including more eigenvector does not bring any improvement in performance or even make the performance worse. Theoretically, we are able to preserve more information by using more eigenvectors. However, noise signals may be also preserved, which can be filtered when using fewer eigenvectors. 
    \item The proposed $\m$ achieves the state-of-the-art or at least comparable performance on all the data sets, which shows the effectiveness of the proposed framework $\m$. 
\end{itemize}

To sum up, $\pooling$ can help learn better graph representation and the proposed framework $\m$ with $\pooling$ can achieve state-of-the-art performance in graph classification task.  
{\footnotesize
\begin{table}
	\begin{center}	
		\caption{Performance comparison.}
		\vskip -1em
		\label{tab:performance}
		\begin{tabular} { c c c c c c c c }
		\hline 
			\multirow{ 2}{*}{Baselines} & \multicolumn{7}{c}{Data sets} \\
			&  ENZYMES & PROTEINS & D\&D &  NCI1 & NCI109& MUTAG   \\	
			\hline			
			GCN &  0.440&   0.740  &   0.759& 0.725  & 0.707 & 0.780 \\ [0.5ex]	
			GraphSage& 0.554& 0.746 &0.766 & 0.732& 0.703& 0.785\\
			SET2SET &  0.380& 0.727 &  0.745  &  0.715 &0.686& 0.764\\[0.5ex]\
			DGCNN &0.410&0.732&0.778 & 0.729 & 0.723&0.788\\
			Diff-pool & 0.636&  0.759  &  0.780 & 0.760  & 0.741&  {\bf 0.806}& \\[0.5ex]

			 \m-1&   {\bf 0.650}&  0.751 & 0.775 &   0.760 &  0.746& 0.801 \\[0.5ex]
			\m-2& 0.645 &  0.754 & 0.770 & 0.767 &  0.748& 0.789 \\[0.5ex]
			\m-3 & 0.645 &  {\bf 0.766} &  {\bf 0.786}  & {\bf 0.770} & {\bf 0.749} & 0.795 \\[0.5ex]

			\hline   
		\end{tabular}
	\end{center}
\end{table}
}
\subsection{Understanding Graph Signals} \label{sec:graph_signals}

In this subsection, we investigate the distribution of the Fourier coefficients on signals in real data. We aim to show that for natural graph signals, most of the information/energy concentrates on the first few Fourier models (or eigenvectors). This paves us a way to only use $H$ filters in $\pooling$. Specifically, given one data set, for each graph ${\mathcal{G}_i}$ with $N_i$ nodes and its associated signal ${\bf X}_i \in \mathbb{R}^{N_i\times d}$, we first calculate the graph Fourier transform and obtain the coefficients $\hat{\bf X}_i \in \mathbb{R}^{N_i\times d}$. We then calculate the following ratio: $r_i^H =   \frac{\|\hat{\bf X}_i[1:H,:]\|_2^2}{\|\hat{\bf X}_i\|_2^2}$,
where $\hat{\bf X}_i[1:H,:]$ denotes the first $i$ rows of the matrix $\hat{\bf X}_i$ for various values of $H$. According to Theorem~\ref{thm:energy_preserving_global}, this ratio measures how much information can be preserved by the first $H$ coefficients. We then average the ratio over the entire data set and obtain
\begin{align}
    r^H = \sum\limits_{i} r_i^H.
\end{align}
Note that if $H>N_i$, we set $r^H_i=1$. We visualize the ratio for each of the data set up to $H=40$ in Figure~\ref{fig:coef_results}. As shown in Figure~\ref{fig:coef_results}, for most of the data set, the magnitude of the coefficients concentrated in the first few coefficients, which demonstrates the reasonableness of using only $H \ll N_{max}$ filters in $\pooling$. In addition, using $H$ filters can save computational cost.

\begin{figure}
	\begin{center}
		\subfigure[ENZYMES]{\label{fig:ENZYMES}\includegraphics[scale=0.26]{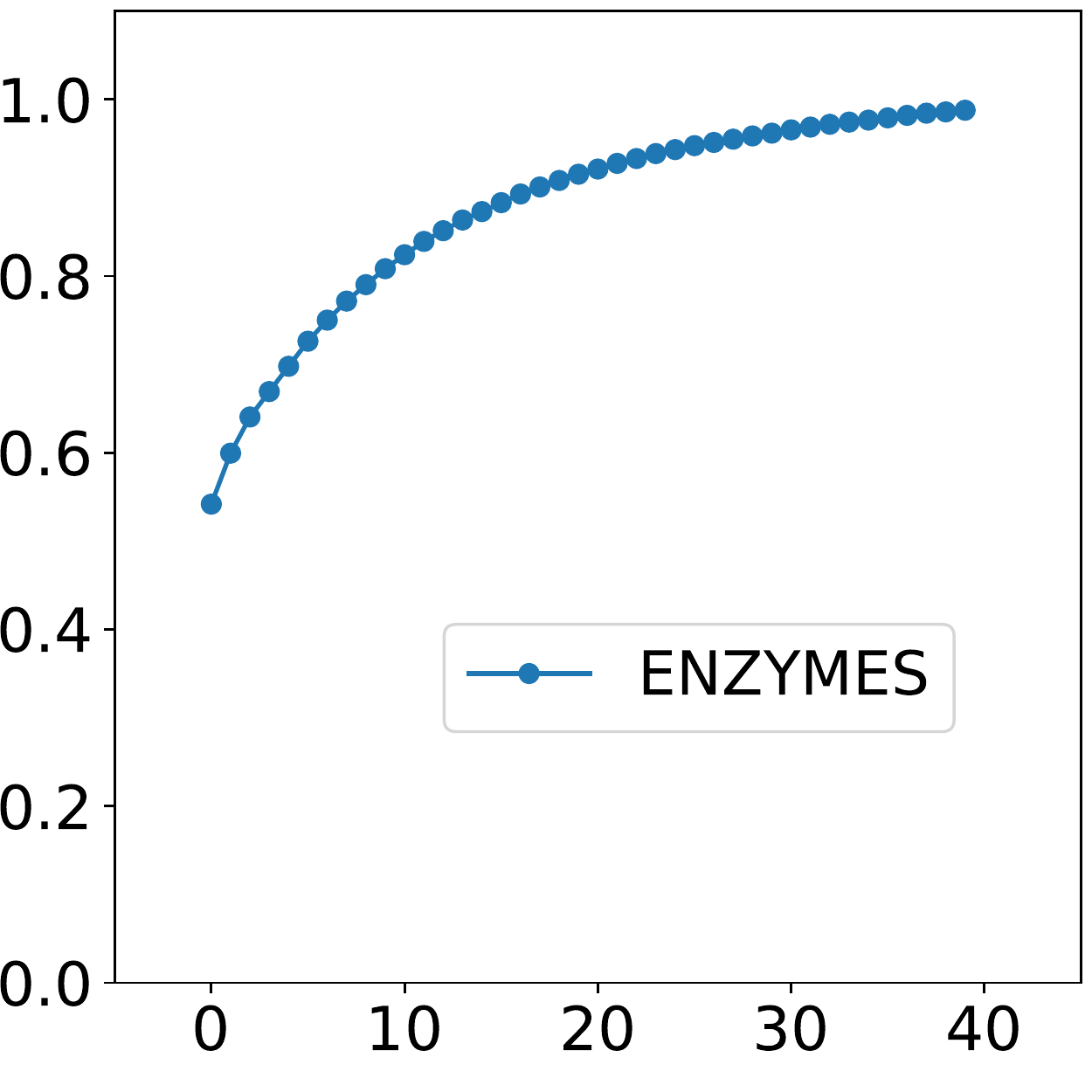}}
		\subfigure[PROTEINS]{\label{fig:Proteins}\includegraphics[scale=0.26]{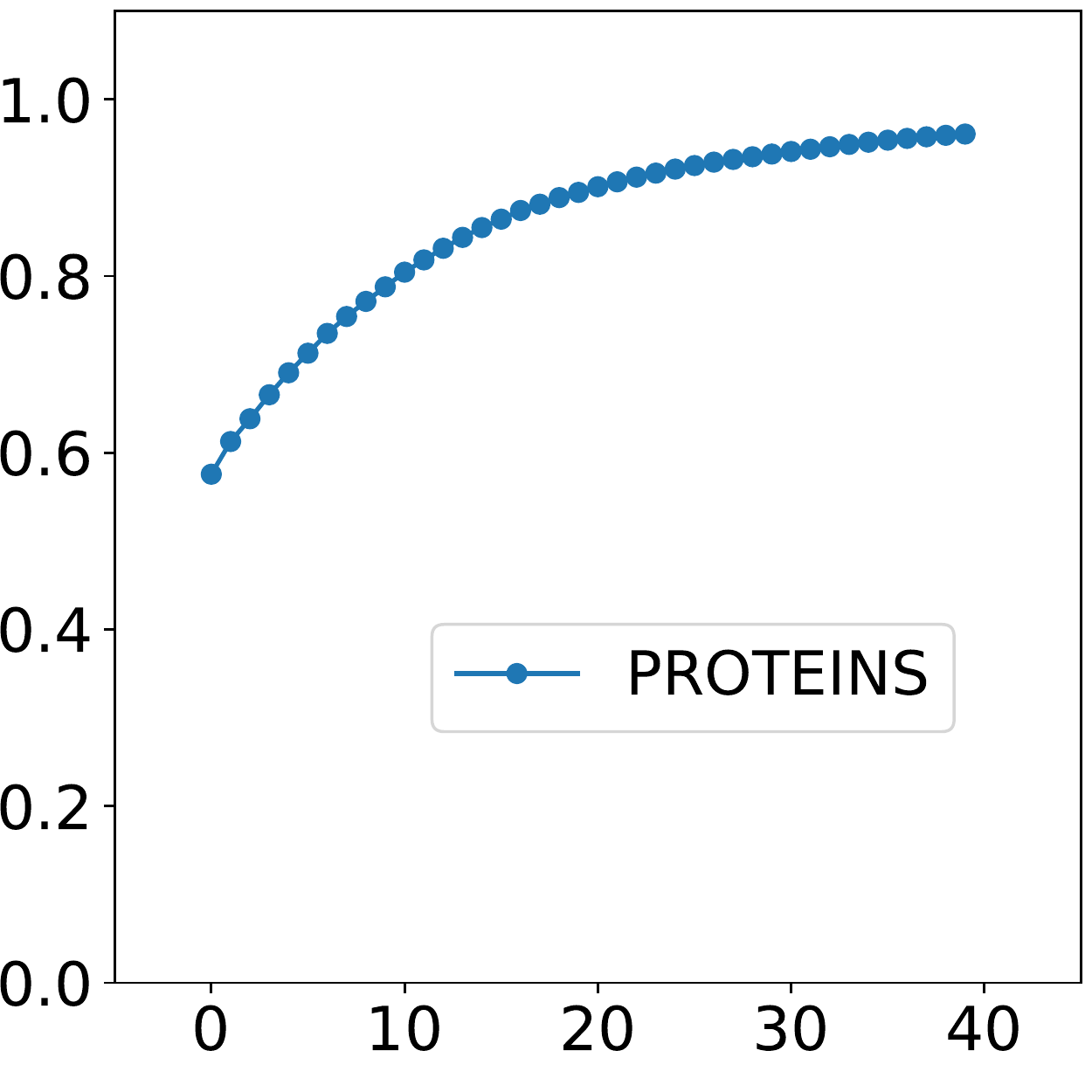}}
		\vskip -1em
		\subfigure[NCI1]{\label{fig:NCI1}\includegraphics[scale=0.26]{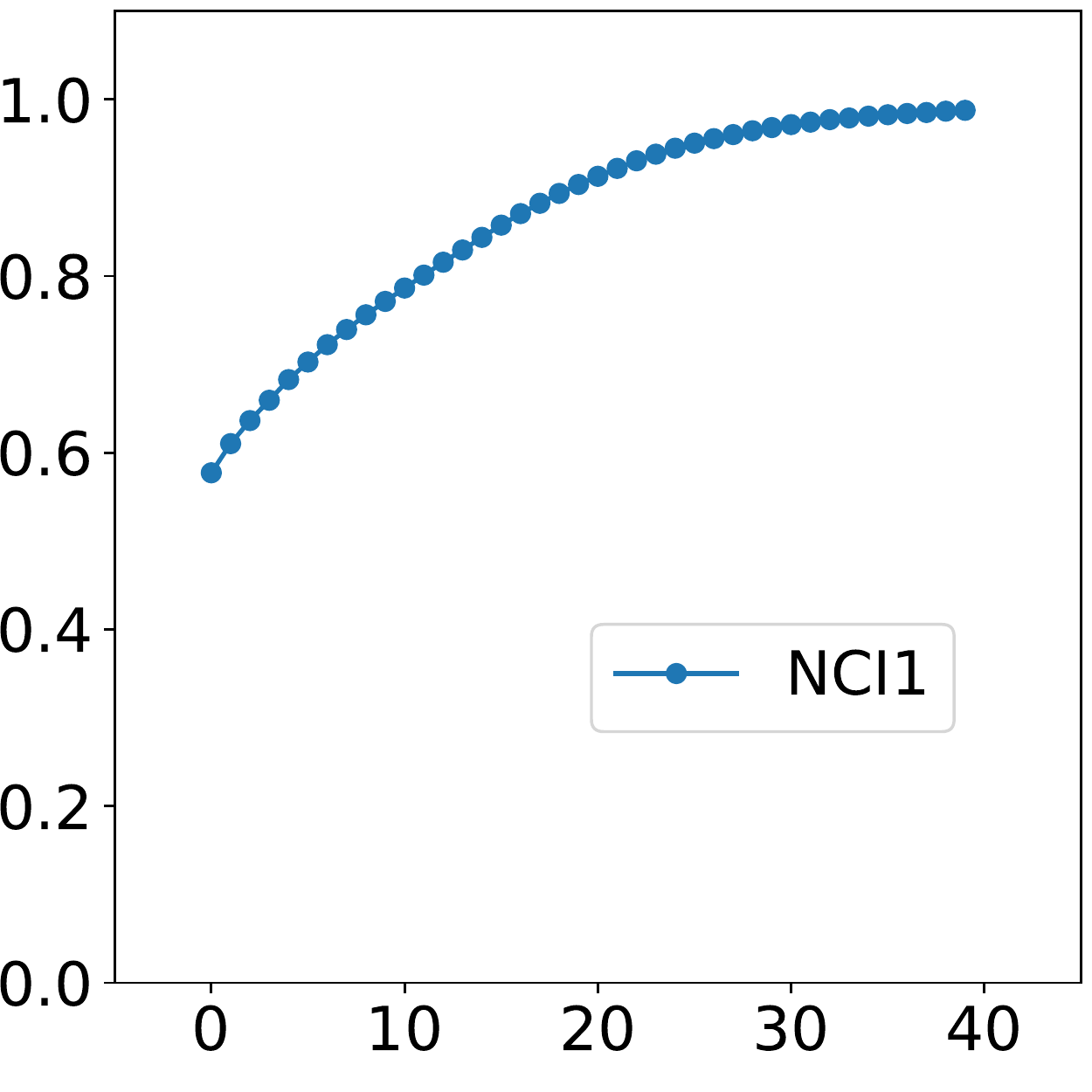}}
		\subfigure[NCI109]{\label{fig:NCI109}\includegraphics[scale=0.26]{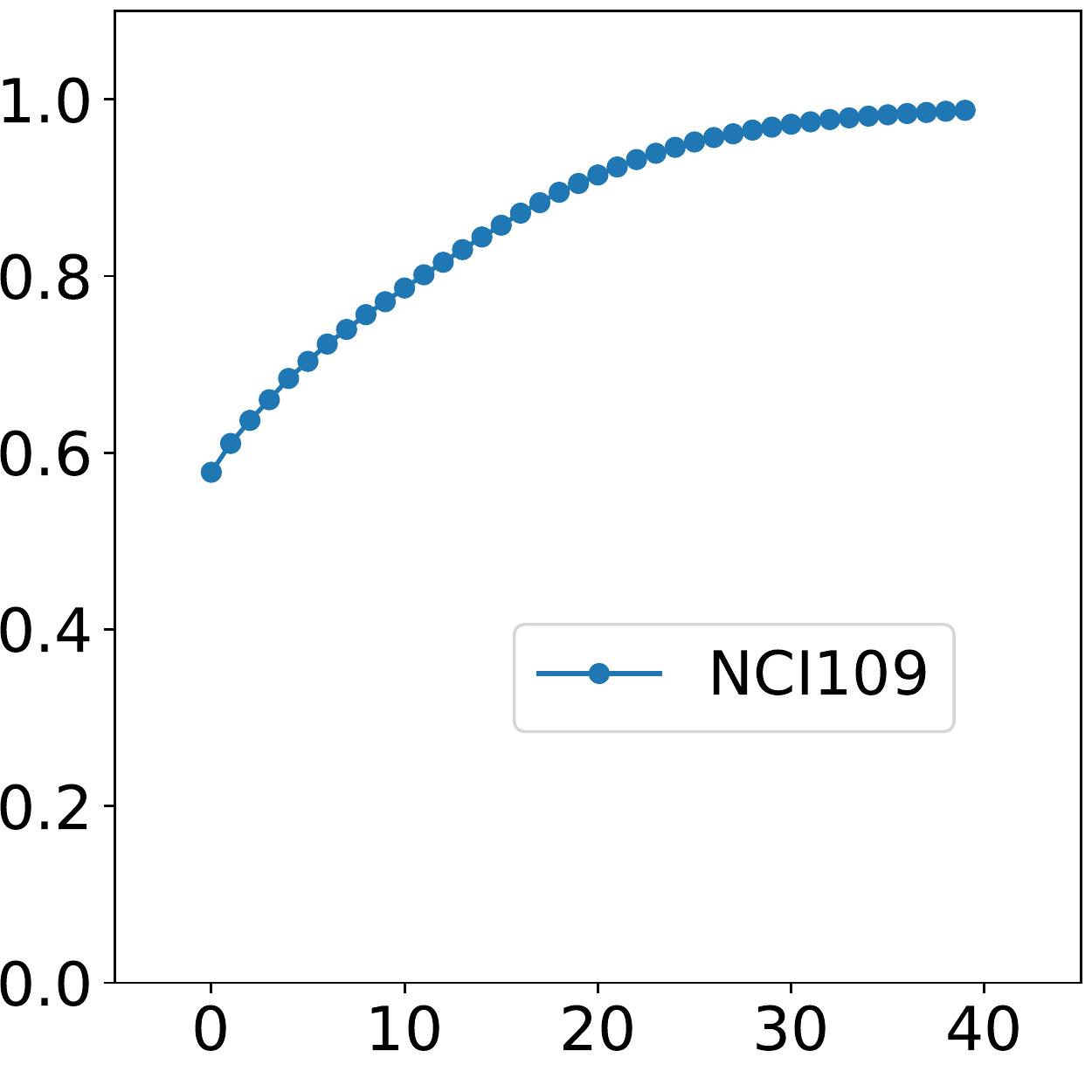}}
        \vskip -1em
		\subfigure[Mutagenicity]{\label{fig:Mutagenicity}\includegraphics[scale=0.26]{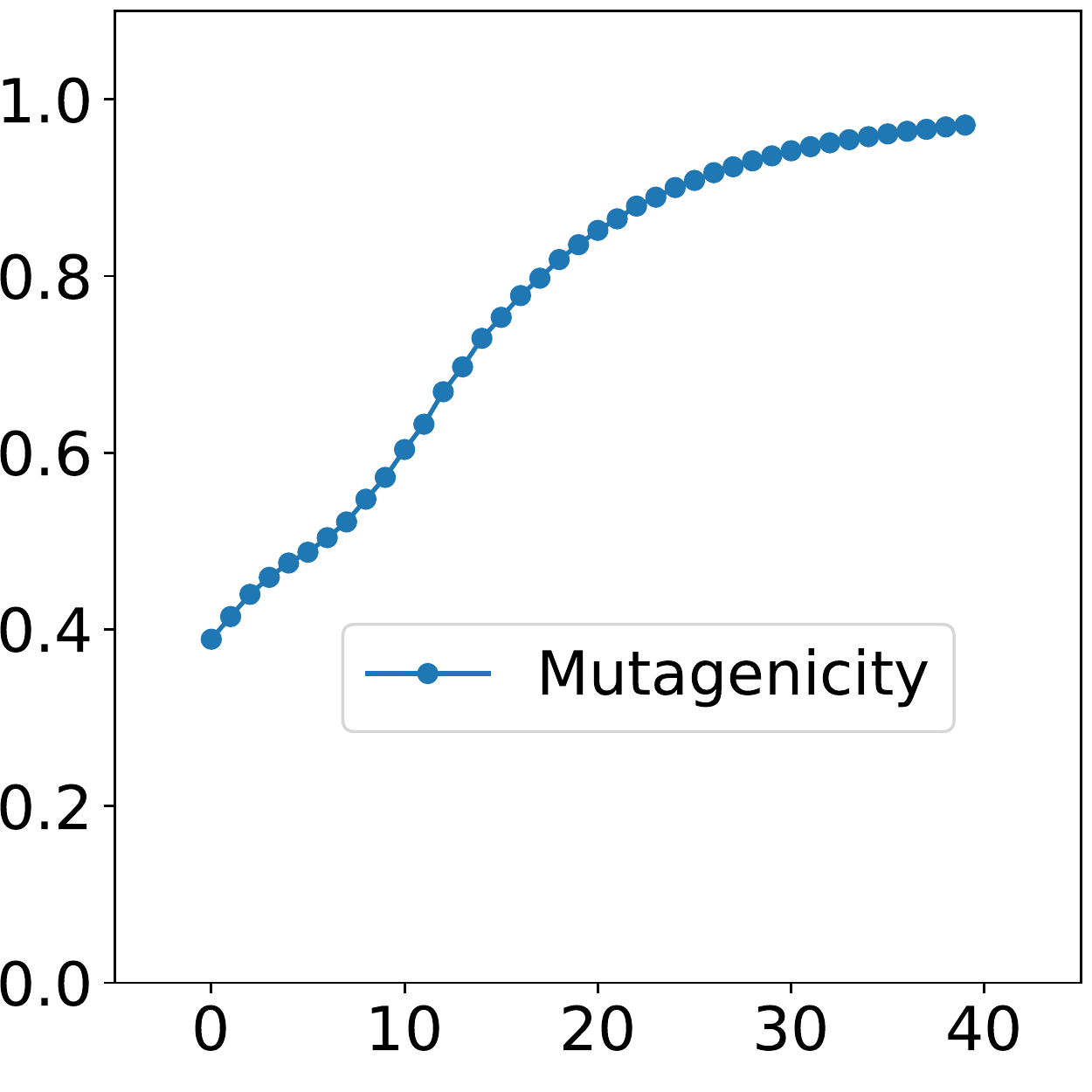}}
		\subfigure[DD]{\label{fig:DD}\includegraphics[scale=0.26]{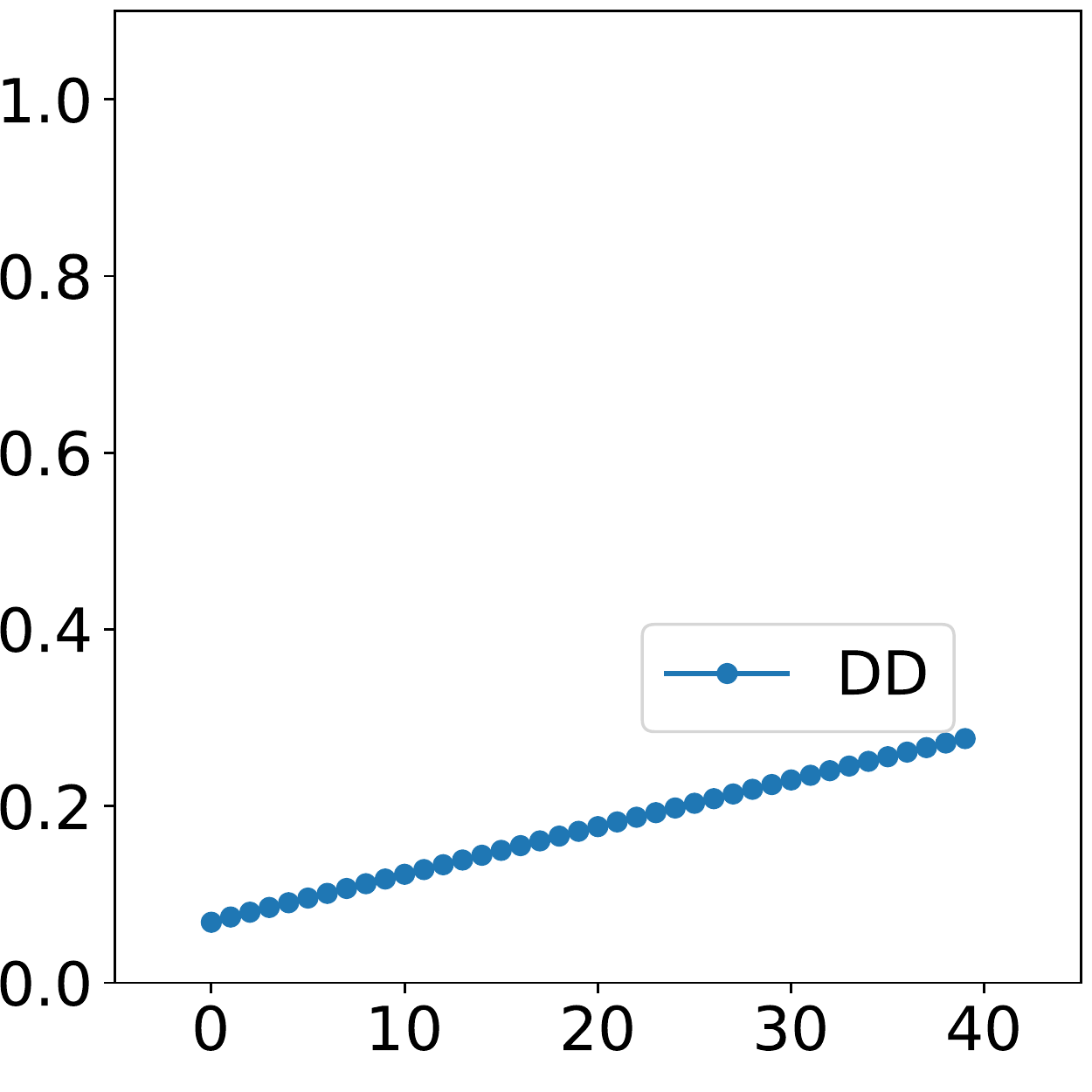}}
	\end{center}
		\vspace{-0.2in}
	\caption{Understanding graph signals}
		\vspace{-0.2in}
	\label{fig:coef_results}
\end{figure}

\section{Related Work} \label{sec:related_work}

In recent years, graph neural network models, which try to extend deep neural network models to graph structured data, have attracted increasing interests. These graph neural network models have been applied to various applications in many different areas. In~\cite{kipf2016semi}, a graph neural network model that tries to learn node representation by aggregating the node features from its neighbors, is applied to perform semi-supervised node classification. Similar methods were later proposed to further enhance the performance by including attention mechanism~\cite{velivckovic2017graph}. GraphSage~\cite{ying2018hierarchical}, which allows more flexible aggregation procedure, was designed for the same task. There are some graph neural networks models designed to reason the dynamics of physical systems where the model is applied to predict future states of nodes given their previous states~\cite{battaglia2016interaction,sanchez2018graph}. Most of the aforementioned methods can fit in the framework of ``message passing'' neural networks~\cite{gilmer2017neural}, which mainly involves transforming, propagating and aggregating node features across the graph through edges. Another stream of graph neural networks was developed based on the graph Fourier transform~\cite{defferrard2016convolutional,bruna2013spectral,henaff2015deep,levie2017cayleynets}. The features are first transferred to the spectral domain,  next filtered with learnable filters and then transferred back to the spatial domain. The connection between these two streams of works is shown in~\cite{defferrard2016convolutional,kipf2016semi}. Graph neural networks have also been extended to different types of graphs~\cite{ma2018dynamic,ma2019multi,derr2018signed} and applied to various applications~\cite{wang2018zero,fan2019graph,ying2018graph,monti2017geometric,schlichtkrull2018modeling,trivedi2017know}.
Comprehensive surveys on graph neural networks can be found in~\cite{zhou2018graph,wu2019comprehensive,zhang2018deep,battaglia2018relational}.

However, the design of the graph neural network layers is inherently ``flat'', which means the output of pure graph neural network layers is node representations for all the nodes in the graph. To apply graph neural networks to the graph classification task, an approach to summarize the learned node representations and generate the graph representation is needed. A simple way to generate the graph representations is to globally combine the node representations. Different combination approaches have been investigated, which include averaging over all node representation as the graph representation~\cite{duvenaud2015convolutional}, adding a ``virtual node'' connected to all the nodes in the graph and using its node representation as the graph representation~\cite{li2015gated}, and using conventional fully connected layers or convolutional layers after arranging the graph to the same size~\cite{zhang2018end,gilmer2017neural}. However, these global pooling methods cannot hierarchically learn graph representations, thus ignoring important information in the graph structure. There are a few recent works~\cite{defferrard2016convolutional,ying2018hierarchical,simonovsky2017dynamic,fey2018splinecnn} investigating learning graph representations with a hierarchical pooling procedure. These methods usually involve two steps 1) coarsen a graph by grouping nodes into supernode to form a hierarchical structure and 2) learn supernode representations level by level and finally obtain the graph representation. These methods use mean-pooling or max-pooling when they generate supernodes representation, which neglects the important structure information in the subgraphs. In this paper, we propose a pooling operator based on local graph Fourier transform, which utilizes the subgraph structure as well as the node features for generating the supernode representations.
\section{Conclusion} \label{sec:conlcusion}
In this paper, we design $\pooling$, a pooling operator based on local graph Fourier transform, which can extract subgraph information utilizing both node features and structure of the subgraph. We provide a theoretical analysis of the pooling operator from both local and global perspectives. The pooling operator together with a subgraph-based graph coarsening method forms the pooling layer, which can be incorporated into any graph neural networks to hierarchically learn graph level representations. We further proposed a graph neural network framework $\m$ by combining the proposed pooling layers with the GCN convolutional layers. Comprehensive graph classification experiments were conducted on $6$ commonly used graph classification benchmarks. Our proposed framework achieves state-of-the-art performance on most of the data sets, which demonstrates its effectiveness. 

\section{Acknowledgements}
Yao Ma and Jiliang Tang are supported by the National Science Foundation (NSF) under grant numbers IIS-1714741, IIS-1715940, IIS-1845081 and CNS-1815636, and a grant from Criteo Faculty Research Award.

\balance
\bibliographystyle{ACM-Reference-Format}
\bibliography{sample-base-abbre.bib}

\end{document}